\documentclass[sigconf,10pt]{acmart}

\AtBeginDocument{%
  }

\usepackage{subfigure}
\usepackage{algorithmic,algorithm}
\usepackage{multirow}
\usepackage{color}
\usepackage{caption}
\usepackage{graphicx}
\usepackage{makecell}
\usepackage{enumitem}
\usepackage{tcolorbox}
\newtheorem{theorem}{Theorem}

\copyrightyear{2024}
\acmYear{2024}
\setcopyright{acmlicensed}\acmConference[ICCAD '24]{IEEE/ACM International
Conference on Computer-Aided Design}{October 27--31, 2024}{New York, NY,
USA}
\acmBooktitle{IEEE/ACM International Conference on Computer-Aided Design
(ICCAD '24), October 27--31, 2024, New York, NY, USA}
\acmDOI{10.1145/3676536.3676669}
\acmISBN{979-8-4007-1077-3/24/10}

\begin{document}

\title{The Power of Graph Signal Processing for Chip Placement Acceleration}
% \author{Anonymous authors}

\author{Yiting Liu}
\affiliation{%
  {\normalsize \department{School of Computer Science}}
  \normalsize \institution{Fudan University}
  \normalsize \city{Shanghai}
  \normalsize \country{China}}
\normalsize \email{yitingliu20@fudan.edu.cn}

\author{Hai Zhou}
\affiliation{%
{\normalsize \department{Department of ECE}}
\normalsize \institution{Northwestern University}
\normalsize  \city{Evanston}
 \normalsize \country{USA}}
\normalsize\email{haizhou@northwestern.edu}

\author{Jia Wang}
\affiliation{%
  {\normalsize \department{Department of ECE}}
 \normalsize \institution{IIIinois Institute of Technology}
 \normalsize \city{Chicago}
 \normalsize \country{USA}
 }
\email{jwang@ece.iit.edu}

\author{Fan Yang}
\affiliation{%
  {\normalsize \department{State Key Lab of ASIC \& System, Microelectronics Department}}
  \normalsize\institution{Fudan University}
 \normalsize \city{Shanghai}
 \normalsize \country{China}}
\normalsize \email{yangfan@fudan.edu.cn}

\author{Xuan Zeng}
\affiliation{%
  {\normalsize \department{State Key Lab of ASIC \& System, Microelectronics Department}}
  \normalsize\institution{Fudan University}
 \normalsize \city{Shanghai}
 \normalsize \country{China}}
\normalsize \email{xzeng@fudan.edu.cn}

\author{Li Shang}
\authornote{Corresponding author.}
\affiliation{%
  {\normalsize \department{School of Computer Science}}
  \normalsize \institution{Fudan University}
 \normalsize \city{Shanghai}
  \normalsize \country{China}}
 \normalsize \email{lishang@fudan.edu.cn}

\begin{abstract}

Placement is a critical task with high computation complexity in VLSI physical design. Modern analytical placers formulate the placement objective as a nonlinear optimization task, which suffers a long iteration time. To accelerate and enhance the placement process, recent studies have turned to deep learning-based approaches, particularly leveraging graph convolution networks (GCNs). However, learning-based placers require time- and data-consuming model training due to the complexity of circuit placement that involves large-scale cells and design-specific graph statistics.

This paper proposes GiFt, a parameter-free technique for accelerating placement, rooted in graph signal processing. 
GiFt excels at capturing multi-resolution smooth signals of circuit graphs to generate optimized placement solutions without the need for time-consuming model training, and meanwhile significantly reduces the number of iterations required by analytical placers. Experimental results show that GiFt significantly improving placement efficiency, while achieving competitive or superior performance compared to state-of-the-art placers. In particular, compared to DREAMPlace, the recently proposed GPU-accelerated analytical placer, GF-Placer improves total runtime over 45\%.
\end{abstract}

\maketitle
\section{Introduction}

Placement is a critical and complex task in VLSI physical designs. It determines the locations of connected cells to minimize total wirelength while adhering to constraints (e.g. density constraints).
Despite continuous efforts to accelerate and improve the placement process, it remains challenging due to the growing complexity of modern designs.
The state-of-the-art analytical placers~\cite{2015eplace, replace} model the placement problem as an electrostatic system and iteratively optimize it, which involves long iteration time.
Recently proposed DREAMPlace~\cite{dreamplace3} has enhanced this process via GPU acceleration. However, long iterations remain a challenge to analytical placers. 

To expedite the placement process, deep learning-based methods have drawn significant attention.
Recently proposed graph convolution network (GCN) based methods represent circuit netlist as graph and produce initial chip placement through efficient model inference. For instance, GraphPlanner~\cite{graphplanner} utilizes a variational GCN-based model to generate initial placement solutions, streamlining the placement procedure for improved efficiency. Researches~\cite{CY2021, google21} apply GCNs and their variants to encode the connectivity information for generating node embeddings. These embeddings are subsequently employed in reinforcement learning calculations to yield superior placement results.

However, existing empirical GCN-based approaches face two key limitations. 
Firstly, these approaches require time-consuming parameter learning with high computation and memory costs.  
Specifically, GCN-based deep models equipped with a large number of learnable parameters demand a significant amount of training data and considerable training time. As the scale of modern designs continues to increase, the computational complexity and memory demands of GCN-based deep models increase accordingly.  
Secondly, it is challenging to guarantee the generalization of GCN-based models. The statistical characteristics of graph structures across various designs can vary significantly, and the existence of predetermined fixed cells (e.g. fixed IOs) exacerbates the issue. Consequently, even well-trained GCNs may struggle to generalize to other unseen designs, limiting their practical applicability. 
Therefore, an interesting question arises: \emph{How to minimize the required optimization iterations of analytical solvers and yet avoid the high training costs of GCN-based methods? }

This paper presents GiFt, an efficient parameter-free approach for accelerating chip placement. GiFt can be seamlessly integrated with modern analytical placers to construct an ultra-fast placement flow GiFt-Placer, which significantly minimizes the numerous optimization iterations of placers without the need for time-consuming model training.
% GiFt has the capability to explore and capture comprehensive structural information of circuit graphs. 
% It can be seamlessly integrated with modern analytical placers to significantly minimize the numerous optimization iterations of placers without the need for time-consuming model training. 
In essence, GiFt is theoretically rooted in graph signal processing (GSP) and we emphasize the crucial role of \emph{smoothness}, a key concept in GSP, in enhancing chip placement. Specifically, GiFt functions as a multi-frequency \emph{Graph Filter} with low computation complexity to promote both local and global signal smoothness on circuit graphs. This feature facilitates the generation of optimized initial placement solutions, which drives the subsequent placers to generate high-quality placement results with highly reduced iteration time.
Moreover, we prove that both the classic eigenvector-based placers~\cite{eigen_placer, eigen_placer2} and recently emerged GCN-based placers~\cite{graphplanner, CY2021} are special cases of GiFt-Placer. Our work further points out that these approaches introduce redundant computations, which are unnecessary for high-quality chip placement. Experimental results on academic benchmarks show that GiFt-Placer significantly improves placement efficiency, while achieving competitive or superior performance compared to state-of-the-art placers. In particular, compared to DREAMPlace, the recently proposed GPU-accelerated analytical placer, GiFt-DREAMPlace improves runtime over 45\%.

The key contributions are listed as follows.

\begin{enumerate}

\item We propose a new multi-frequency graph filter-based placement acceleration approach called GiFt, which can comprehensively capture the graph structure and efficiently generate optimized cell locations. 

\item GiFt offers high efficiency through sparse matrix multiplication and does not require a time-consuming model training process due to its parameter-free nature.

\item GiFt can be seamlessly integrated with modern analytical placers to construct GiFt-Placer placement flow, which effectively reduces the number of iterations required for analytical placers and significantly improves placement efficiency, even surpassing the performance of analytical placers running on GPU versions.

\item By examining classic eigenvector-based placers and recent GCN-based placers from the perspective of graph signal processing, we prove that they can be considered as special cases of GiFt-Placer, but they introduce unnecessary computation costs to chip placement.

% Experimental results on both academic benchmarks and real-world industry designs show GF-Placer achieves competitive or superior performance compared to both GCN-based placer and the analytical placer running on GPU version while remarkably improving the placement efficiency.  

\end{enumerate}

The rest of this paper is organized as follows.  Section~\ref{sec:pre} describes preliminaries for the rest of this paper. Section~\ref{sec:method} introduces the details of the proposed approach. Section~\ref{sec:experiment} presents the experimental results. Section~\ref{sec:discussion} discusses the theoretical foundations of this work. We conclude the paper in Section~\ref{sec:conclusion}.

\section{Preliminaries} \label{sec:pre}
This section introduces the fundamental knowledge of graph signal processing (GSP) and global placement.

\subsection{Graph Signal Processing}~\label{sec:gsp}
Given a weighted undirected graph $G = (V, E)$ with a set of nodes $V$ and edges $E$, it can be represented as an adjacency matrix $A =\{w_{i,j}\}\in \mathbb{R}^{N \times N}$ with $w_{i,j}>0$ if node $v_i$ and $v_j$ are connected by edges, and $w_{i,j}=0$ otherwise. The graph Laplacian matrix is defined as $L = D-A$, where $D =diag(d_1, d_2, \ldots, d_n )\in \mathbb{R}^{N \times N}$ represents the degree matrix of $A$ with $d_i=\sum_{v_j \in V}w_{i,j}$. The normalized graph Laplacian matrix is defined as $\tilde{L}=D^{-\frac{1}{2}}LD^{-\frac{1}{2}}$.

The graph signal $g$ is defined as a mapping $g : V \to \mathbb{R}$ and it can be expressed as an n-dimensional vector where each element $g_i$ can encapsulate diverse types of information associated with node $v_i$~\cite{graph_signal}. 
GSP is a field dedicated to the analysis and manipulation of these signals on graphs, intending to extract meaningful insights or accomplish various tasks.

\textbf{Smoothness.}
The graph gradient of signal $g$ at node $v_i$ is defined as $\nabla_ig:=[\{\frac{\partial f}{\partial e}|_i\}_{e \in E} ]=[\sqrt{w_{i,j}}(g_i-g_j)].$
% The edge derivative of a signal $g$ with respect to edge $e=(i, j)$ at node $v_i$ is defined as $\frac{\partial f}{\partial e}|_i=\sqrt{A_{i,j}}(g_i-g_j)$ and 
The smoothness of the graph signal can be quantified using the graph Laplacian quadratic form, as calculated by the following equation:
\begin{equation}\label{eq:smoothness}
S(g)=\frac{1}{2}\sum_{v_i\in V}||\nabla_ig||_2^2=\sum_{(v_i,v_j)\in E}w_{i,j}(g_j-g_i)^2
\end{equation}
% The normalized form can be calculated using the Rayleigh quotient as follows:
% \begin{equation}\label{eq:smoothness_ray}
% R(g)=\frac{S(g)}{||g||_2}=\frac{\sum_{(v_i,v_j)\in E}w_{i,j}(g_j-g_i)^2}{\sum_{v_i \in V}g_i^2}=\frac{g^TLg}{g^Tg}
% \end{equation}

\textbf{Graph Fourier Transform.}
Since the graph Laplacian matrix $L$ is real and symmetric, it can be decomposed into $L=U\Lambda U^T$, where $\Lambda=diag(\lambda_1, \ldots, \lambda_n) $ denotes eigenvalues with $\lambda_1 \leq \lambda_2, \ldots, \leq \lambda_n$ and $U=[u_1, \ldots, u_n]$ is the corresponding eigenvectors. Using eigenvector matrix $U$ as the Graph Fourier Transform (GFT) basis, we call $\hat{g}=U^Tg$ as GFT that transforms the graph signal $g$ from the spatial domain to the spectral domain. And its inverse transform is defined as $g = U\hat{g}$. 

% Next, we interpret the relation between smoothness and frequency of graph signal. 
% Based on GFT, the measurement of smoothness in Eq.~\ref{eq:smoothness_ray} can be transformed as follows:
% \begin{equation}\label{eq:smoothness_ray}
% R(g)=\frac{g^TLg}{g^Tg}=\frac{g^TU\Lambda U^Tg}{g^TUU^Tg}=\frac{\hat{g}^T\Lambda\hat{g}}{\hat{g}^T\hat{g}}=\frac{\sum_{v_i \in V} \lambda_i\hat{g_i}^2}{\sum_{v_i \in V}\hat{g_i}^2}
% \end{equation}

% From Eq.~\ref{eq:smoothness_ray}, it's evident that $R(g_i)=\lambda_i$. This illustrates that graph signals associated with lower eigenvalues (i.e. lower frequencies) exhibits a greater degree of smoothness.

\textbf{Graph Filter.}
In the spectral domain, the graph filter can be used to selectively filter some unwanted frequencies present in the graph signal, which is formally defined as:
\begin{equation}\label{eq:filter}
\mathcal{H}=Udiag(h(\lambda_1),h(\lambda_2), \ldots, h(\lambda_n))U^T,
\end{equation}
where $h(\cdot)$ is a filter function applied to frequencies.
 % (i.e., eigenvalues).

% Usually, non-smooth high frequency components in signals contains more noises. To extract useful information, low-pass filters are needed to save low frequencies in signals and promote the smoothness of graph signal for denoising. The graph low-pass filter are defined as follows.

% \textbf{Graph Convolution.}
% The graph convolution of a graph signal $g$ is defined as follows:
% \begin{equation}\label{eq:graph_convolution}
% \mathcal{H}=Udiag(h(\lambda_1),h(\lambda_2), \ldots, h(\lambda_n))U^Tg.
% \end{equation}
% The process of graph convolution can be interpreted from the graph signal processing perspective. In essence, given a graph signal $g$, it is initially transformed from spatial domain to spectral domain via the Graph Fourier Transform.  
% Subsequently, unwanted frequencies within the signal are removed using the filter $h(\cdot)$, and finally, the filtered signal is transformed back into the spatial domain through the Inverse Graph Fourier Transform.

\subsection{Global Placement}

A placement instance can be formulated as a graph $G=(V, E)$ with a set of objects $V$ (e.g. IOs, macros, and standard cells) and edges $E$.
The main objective of placement $f(g)$ (see Eq.~\ref{eq:global_placement_eq}) is to find a solution $g \in \mathbb{R}^{N \times 2} $ with minimized total wirelength $S(g)$ subject to density constraints (i.e., the density $\rho_b(g)$ does not exceed a predetermined density $\rho_t$), where $N$ is the number of objects.

\begin{equation}\label{eq:global_placement_eq}
\begin{aligned}
&\min f(g) = S(g) \quad\quad s.t. \quad \rho_b(g) \leq \rho_t. 
\end{aligned}
\end{equation}

The total wirelength $S(g)$ can be estimated as the weighted sum of squared distances between connected objects:

\begin{equation}\label{eq:wire_eq}
\begin{aligned}
S(g) = \sum_{(v_i,v_j)\in E} w_{i,j}((x_i- x_j)^2 + (y_i- y_j)^2).
\end{aligned}
\end{equation}
Considering one net with $M$ pins, we set the weight $w = \frac{2}{M}$.

The placement area can be evenly partitioned into a series of grids (bins)~\cite{2015eplace}. The density of each grid $b$ is computed using the formula in Eq.~\ref{eq:density_eq}:
\begin{equation}\label{eq:density_eq}
\rho_b(g) = \sum_{v \in V} l_x(v, b)l_y(v, b)
\end{equation}
where $l_x(v, b)$ and $l_y(v, b)$ quantify the horizontal and vertical overlaps between the grid $b$ and the object $v$.

\section{Method} \label{sec:method}
In this section, we first explore the placement process from the perspective of graph signal processing and highlight the importance of smoothness for chip placement.
Then we present GiFt, the parameter-free placement acceleration approach, and demonstrate its efficacy via theoretical analysis. Finally, we prove that both the classic eigenvector-based placer~\cite{eigen_placer, eigen_placer2} and recently emerged GCN-based placer~\cite{graphplanner, CY2021} are special cases of the proposed approach, while both of them introduce extra computations and complexity.

\begin{figure*}[htbp]
\includegraphics[width=\textwidth]{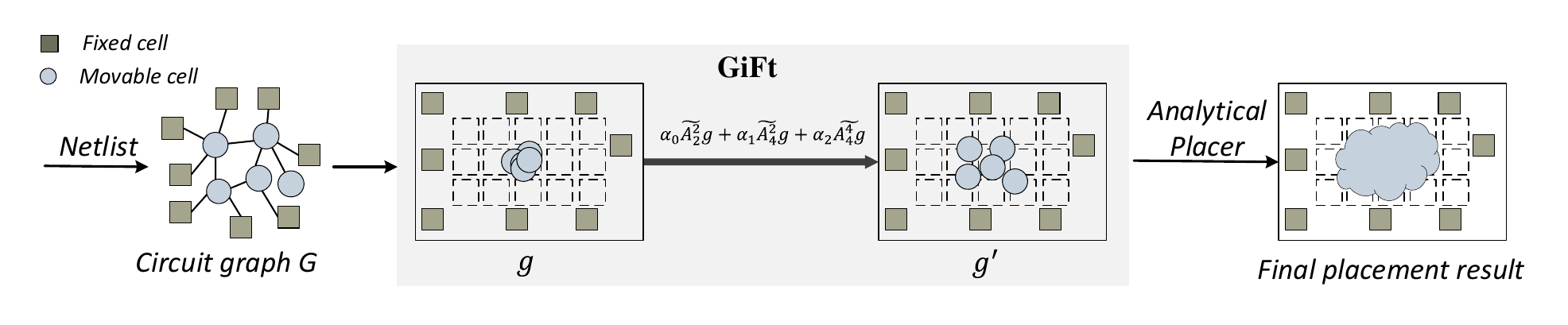}
\centering  
\vspace{-2.5em}
\caption{The workflow of GiFt-equipped placement process.}\label{fig:workflow}
\vspace{-1.5em}
\end{figure*}

\subsection{Enhancing Placement from GSP Perspective}~\label{sec:method1}
As demonstrated in Eq.~\ref{eq:smoothness} and Eq.~\ref{eq:wire_eq}, the optimization objective in placement, which aims to minimize quadratic wirelength, aligns with the goal of enhancing graph signal smoothness from the perspective of GSP. Inspired by this finding, we reframe the placement problem as an effort to improve the smoothness of graph signals on given circuit graphs.
Next, we will explain how we can improve the smoothness of graph signals across the graph.

The normalized form of the smoothness measurement (as depicted in Eq.~\ref{eq:smoothness}) can be calculated using the Rayleigh quotient, as shown below:
\begin{equation}\label{eq:smoothness_ray}
R(g)=\frac{S(g)}{||g||_2^2}=\frac{\sum_{(v_i,v_j)\in E}w_{i,j}(g_j-g_i)^2}{\sum_{v_i \in V}g_i^2}=\frac{g^TLg}{g^Tg}.
\end{equation}
The smaller value of $R(g)$ indicates higher smoothness of graph. 
Leveraging the GFT, we can transform Eq.~\ref{eq:smoothness_ray} as follows:
\begin{equation}\label{eq:smoothness_ray2}
R(g)=\frac{g^TLg}{g^Tg}=\frac{g^TU\Lambda U^Tg}{g^TUU^Tg}=\frac{\hat{g}^T\Lambda\hat{g}}{\hat{g}^T\hat{g}}=\frac{\sum_{v_i \in V} \lambda_i\hat{g_i}^2}{\sum_{v_i \in V}\hat{g_i}^2}.
\end{equation}
It is evident from Eq.~\ref{eq:smoothness_ray2} that $R(g_i)=\lambda_i$, where $g_i$ denotes the graph signal defined on node $v_i$, and $\lambda_i$ denotes $i$-th eigenvalue of the graph Laplacian. This demonstrates that graph signals associated with lower eigenvalues (i.e., lower frequencies) exhibit a higher degree of smoothness.
Consequently, it becomes essential to filter out undesired frequencies, particularly the high-frequency components in graph signals, to enhance the smoothness of these signals, thereby obtaining optimized placement solutions.

Building on this insight, we present our \textbf{placement framework through the lens of GSP}:

Given the input graph signals $g \in \mathbb{R}^{N \times 2}$, which can represent the initial cell locations, potentially containing unwanted noises, we apply a graph filter $\mathcal{H}$ to $g$ to filter out undesired frequencies and obtain the filtered signals $g^{\prime} \in \mathbb{R}^{N \times 2}$, which represents the cell locations processed by the graph filter. The framework can be expressed as follows:
\begin{equation}\label{eq:unified_framework}
g^{\prime}=\mathcal{H}g=Udiag(h(\lambda_1),h(\lambda_2), \ldots, h(\lambda_n))U^Tg,
\end{equation}
where $h(\cdot)$ is a filter function defined on eigenvalues.
This process can be interpreted from a GSP perspective: Given a graph signal $g$, it is initially transformed from the spatial domain to the spectral domain via the GFT.  
Subsequently, unwanted frequencies within the signal are removed using the filter $h(\cdot)$, and finally, the filtered signal is transformed back into the spatial domain through the inverse GFT.

Since the smoothness of graph signals is closely linked to the minimization of placement wirelength, a straightforward approach is to develop an ideal low-pass filter that directly eliminates all high-frequency signals. This can be formulated as follows,
\begin{equation}\label{eq:ideal_filter}
h(\lambda_i)=\left\{
\begin{aligned}
 1,\ &\text{if} \ \lambda_i<\lambda_t\\
 0,\ &\text{otherwise} \\
\end{aligned}
\right.
\end{equation}
where $\lambda_t$ denotes the cut-off frequency. 
Then the filtered signals $g^{\prime}$ are given by 
\begin{equation}\label{eq:unified_framework}
g^{\prime}=\mathcal{H}g=Udiag(\overbrace{1, \ldots, 1}^t, \overbrace{0,\ldots,0}^{n-t})U^Tg=U_tU_t^Tg.
\end{equation}

Although this approach guarantees the global smoothness of graph signals, it requires eigendecomposition to obtain the eigenvectors corresponding to the first $k$ lowest eigenvalues, which is a computationally expensive task for large-scale circuits.
In addition, it is essential to note that globally smooth signals overlook valuable local information, which can be characterized as globally high-frequency yet locally smooth~\cite{LLG2023}. This oversight may lead to an over-smoothing issue, where neighboring locations become too similar, resulting in high overlap in local areas and violations of density constraints.

% In addition, it is essential to note that globally smooth signals overlooks valuable local information.
% This local information can be characterized as globally high-frequency yet locally smooth. Disregarding these signals may lead to an over-smoothing issue, meaning that neighboring locations become too similar, resulting in high overlap in local areas and violations of density constraints. 

To this end, we need to thoughtfully devise a new graph filter that is computationally efficient while considering multi-resolution smooth signals, thereby yielding optimized cell locations $g^{\prime}$ that have minimized total wirelength (representing smoothness over the graph structure) and reduced overlap (i.e., avoid over-smoothing).

\subsection{GiFt: An Efficient Placement Speedup Technique}~\label{sec:method2}
This section introduces GiFt, a GSP-based placement acceleration approach. GiFt functions as an efficient graph filter, utilizing multi-frequency graph signals for comprehensive graph structural analysis and the generation of optimized cell locations. It can be seamlessly integrates with analytical placers to produce high-quality placement solutions while significantly reducing placement time. The GiFt-equipped placement process is depicted in Figure~\ref{fig:workflow}.
% Notably, this approach does not have time-consuming training costs associated with recently emerged deep learning-based approach, while also significantly reducing the number of iterations required by classic analytical placers.

\subsubsection{The algorithm of GiFt.}~\label{sec:gift}
This section presents the theoretical underpinnings of GiFt.
From the perspective of GSP, the normalized adjacency matrix $\tilde{A}$ of the given circuit graph corresponds to the filter function $h(\lambda)=1-\lambda$. The theoretical proof is as follows.
\begin{theorem}~\label{proof:adj_filter}
The normalized adjacency matrix $\tilde{A}=D^{-\frac{1}{2}}AD^{-\frac{1}{2}}$ is a graph filter corresponding to the filter function $h(\lambda)=1-\lambda$.
\end{theorem}
 
\begin{proof}
As the eigendecomposition of the normalized graph Laplacian is given by $\tilde{L}=U\Lambda U^T=Udiag(\lambda_1,\lambda_2, \ldots, \lambda_n)U^T$, $\tilde{L}$ corresponds to the filter function $h(\lambda)=\lambda$.

Since we have $L=D-A$, then
\begin{equation}\label{eq:eq1}
\tilde{L}=D^{-\frac{1}{2}}LD^{-\frac{1}{2}}=D^{-\frac{1}{2}}(D-A)D^{-\frac{1}{2}}=I-D^{-\frac{1}{2}}AD^{-\frac{1}{2}}
\end{equation}

Let $u_i$ denote the eigenvector corresponding to the eigenvalue $\lambda_i$, and $Lu_i=\lambda_i u_i$, we have
\begin{equation}\label{eq:eq3}
Lu_i=(I-D^{-\frac{1}{2}}AD^{-\frac{1}{2}})u_i=\lambda_i u_i.
\end{equation}
It can be transformed to 
\begin{equation}\label{eq:eq4}
D^{-\frac{1}{2}}AD^{-\frac{1}{2}}u_i=(1-\lambda_i)u_i
\end{equation}
which indicates that $\tilde{A}$ is a graph filter corresponding to the filter function $h(\lambda)=1-\lambda$, that is, 
\begin{equation}\label{eq:eq4}
\tilde{A}=U\Lambda_A U^T=Udiag(1-\lambda_1,1-\lambda_2, \ldots, 1-\lambda_n)U^T
\end{equation}
\end{proof}

\begin{figure*}[htbp]
\includegraphics[width=0.88\textwidth]{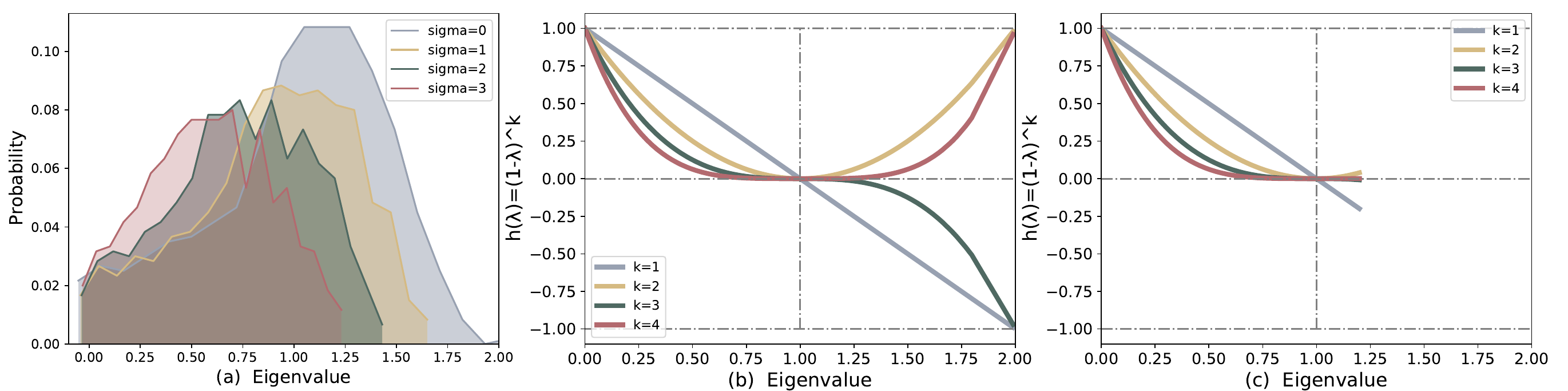}
\centering  
\vspace{-0.5em}
\caption{Eigenvalue distributions of normalized Laplacians with different graph filters: (a) under different self-loops. (b) with $\tilde{A}$ filter at various $k$, (c) with augmented $\tilde{A}$ filter at various $k$.}\label{fig:tmp}
% (b) The effect of using $\tilde{A}$ for graph filtering under different $k$. (c) The effect of using augmented $\tilde{A}$ for graph filtering under different $k$.
\vspace{-1em}
\end{figure*}

Since the eigenvalues of the normalized Laplacian fall within the interval of $[0,2]$, $\tilde{A}$ acts as a band-stop filter that attenuates intermediate-frequency components, which is not effective for denoising graph signals. 

To address this problem, we introduce two enhancements to $\tilde{A}$ to boost its denoising ability.
Firstly, we add self-loops to $\tilde{A}$ (called augmented $\tilde{A}$) to shrink high-frequency components (e.g. large eigenvalues)~\cite{sgc}, thereby removing high-frequency noises.
As depicted in Figure~\ref{fig:tmp}(a), by adding self-loops as follows $A=A+\sigma I$ (where $\sigma=0, 1, 2, 3$ in Figure~\ref{fig:tmp}(a)), the large eigenvalues become smaller, leading the augmented $\tilde{A}$ to perform like a low-pass filter. 
It is important to note that if $\sigma$ is too large, most eigenvalues will approach zero, rendering the graph filter less effective in noise removal.
Secondly, we generalize the augmented $\tilde{A}$ to $\tilde{A}^k$. This modification transforms its corresponding filter function from $h(\lambda)=1-\lambda$ to $h_k(\lambda)=(1-\lambda)^k$, where $k$ controls the strength of the graph filter. Figure~\ref{fig:tmp}(b) and Figure~\ref{fig:tmp}(c) show the filter strength associated with $\tilde{A}$ and augmented $\tilde{A}$ for various values of $k$. As $k$ increases, the low-pass filtering effect of the augmented $\tilde{A}$ becomes more pronounced.
The underlying dataset for Figure~\ref{fig:tmp} is mgc\_edit\_dist\_1 benchmark in the ISPD2014 benchmark suite~\cite{ispd2014}. 
To this end, by adjusting the values of $\sigma$ and $k$, we can effectively regulate the extent to which high-frequency signals are filtered out, ultimately achieving smooth signals across multiple resolutions.

Building upon this analysis, we propose \textbf{GiFt, functioning as multi-frequency graph filters}, to generate optimized cell locations as follows, 
\begin{equation}\label{eq:proposed_filter}
g^{\prime}=GiFt(g)=\alpha_0\tilde{A}_2^2g+\alpha_1\tilde{A}_4^2g+\alpha_2\tilde{A}_4^4g.
\end{equation}
Here, $\tilde{A}_1^2$, $\tilde{A}_2^2$ and $\tilde{A}_4^3$ all function as low-pass filters with varying degrees of filtering strength, which are formulated as follows,
\begin{equation}\label{eq:A_formulation}
\begin{aligned}
    \tilde{A}_2^2=&{((D+2I)^{-\frac{1}{2}}(A+2I)(D+2I)^{-\frac{1}{2}})}^2 \\
    \tilde{A}_4^2=&{((D+4I)^{-\frac{1}{2}}(A+4I)(D+4I)^{-\frac{1}{2}})}^2 \\
    \tilde{A}_4^4=&{((D+4I)^{-\frac{1}{2}}(A+4I)(D+4I)^{-\frac{1}{2}})}^4 \\
\end{aligned}
\end{equation}
To differentiate, we refer to $\tilde{A}_2^2$ as a \emph{high-pass filter} that allows some relatively high-frequency signals to pass through to capture local information. $\tilde{A}_4^2$ corresponds to a \emph{medium-pass filter}, and $\tilde{A}_4^4$ represents a \emph{low-pass filter} designed to exclusively preserve low-frequency signals, thereby enhancing global smoothness.
$\alpha_0$, $\alpha_1$ and $\alpha_2$ are weight coefficients that determine the proportion of globally smooth signals and locally smooth signals.
$g$ symbolizes the input signals, representing the initial cell locations.
$g^{\prime}$ denotes denoised graph signals, i.e., cell locations produced by GiFt.

\subsubsection{The beneficial attributes of GiFt for placement.}
The input to GiFt is initial cell locations denoted by $g$ in Eq.~\ref{eq:proposed_filter}. We can simply set the locations of movable cells at the center of the placement region, following a Gaussian distribution $N(0,1)$, while fixed cells are placed at their predetermined locations. 
Next, we prove that GiFt can generate optimized locations that both consider the predetermined fixed locations and ensure smoothness across the graph.

% \begin{proof}~\label{proof:incremental}
% To obtain the optimized cell locations that take into account both the predetermined fixed cell locations and the smoothness over the graph, 
To achieve this goal, the optimization objective can be formulated as follows:
\begin{equation}\label{eq:incremental}
min\{||g-g^{\prime}||_{2}^{2}+tr(g^{\prime T}Lg^{\prime})\},
\end{equation}
where $g$ denotes the initial cell locations containing predetermined fixed cell locations, $g^{\prime}$ denotes the predicted cell locations, and $L$ denotes the graph Laplacian matrix. 
The first term enforces that the predicted cell locations account for the fixed cell locations, while the second term contributes to the smoothness of the graph.

By setting the derivative of Eq.~\ref{eq:incremental} to zero, we arrive at the following expression:
\begin{equation}\label{eq:derivative}
g^{\prime}=(I+L)^{-1}g
\end{equation}
Since $L$ associates with the filter function $\lambda$ as proved in Theorem~\ref{proof:adj_filter}, $(I+L)^{-1}$ can be transformed into the spectral domain. This transformation yields the graph filter:
\begin{equation}\label{eq:spectral}
h(\lambda)=(1+\lambda)^{-1}
\end{equation}
Then we approximate $h(\lambda)$ using its first-order Taylor expansion: 
\begin{equation}\label{eq:taylor}
\hat{h}(\lambda)=1-\lambda.
\end{equation}
By transforming $\hat{h}(\lambda)$ back into the spatial domain, we obtain:
\begin{equation}\label{eq:final}
\hat{h}(L)=I-L=D^{-\frac{1}{2}}AD^{-\frac{1}{2}}=\tilde{A}.
\end{equation}
To further enhance the filtering capability, we can add self-loops $\sigma$ and a scaled coefficient $k$ into Eq.~\ref{eq:final}. 
Finally, the predicted cell locations can be obtained by
\begin{equation}\label{eq:final}
g^{\prime}=\tilde{A}_{\sigma}^kg,
\end{equation}
This forms the basis of GiFt.

To sum up, GiFt can capture smooth signals across multiple resolutions, effectively preventing over-smoothing. Simultaneously, it can take the design-specific predetermined locations of fixed cells into account. As a result, it enables comprehensive exploration of the given circuit graphs to generate optimized placement solutions.

\subsubsection{The workflow of GiFt-equipped placement process.}
Since density constraints are not rigorously enforced in Eq.~\ref{eq:proposed_filter}, the predicted solutions $g^{\prime}$ can undergo further refinement through the subsequent placer. By integrating GiFt with the analytical placer, we develop the ultra-fast placement flow GiFt-Placer.

% \vspace{-0.5em}
\begin{algorithm}[]
    \renewcommand{\algorithmicrequire}{\textbf{Input:}}
    \renewcommand{\algorithmicensure}{\textbf{Output:}}
    \small \caption{GiFt-Placer Algorithm} \label{alg:gift_placement}
    \footnotesize
    \begin{algorithmic}[1]
        \REQUIRE A circuit netlist modeled by an undirected weighted graph $G = (V, E)$
        \ENSURE  Legalized placement result,
        % \STATE  Transform the netlist to an undirected weighted graph $G = (V, E)$ using clique net model
        \STATE Set initial locations $g$ of cells $v \in V$. Locations of movable cells $g_m \sim N(0,1)$ are placed at the center of the placement region. Fixed cells are set at their predetermined locations.
        \STATE Compute optimized cell locations $g^{\prime}$ using GiFt: $g^{\prime}=\alpha_0\tilde{A}_2^2g+\alpha_1\tilde{A}_4^2g+\alpha_2\tilde{A}_4^4g$
        \STATE Input $g^{\prime}$ to subsequent placer to complete placement task
        \STATE \textbf{return} The legalized placement result
    \end{algorithmic}
\end{algorithm}
% \vspace{-0.5em}

The algorithm of GiFt-Placer is summarized in Alg.~\ref{alg:gift_placement}.
Given a circuit netlist, we use clique net model~\cite{clique} to convert it into a weighted graph $G$. The graph is then processed by GiFt to generate optimized cell locations $g^{\prime}$ (Line 1-2). $g^{\prime}$ is then fed into the analytical placer, serving as the starting point for placement optimization, to complete the placement task with substantially reduced iteration time (Line 3-4).

Next, we elaborate on the power of GiFt as follows.

\textbf{Reduced Iteration Count. }
As GiFt can capture comprehensive graph structures from globally smooth signals to locally smooth ones, it can produce optimized cell locations with reduced overlap. By integrating the predicted locations with analytical placers, the placement process benefits from significantly reduced iterations, as the locations predicted by GiFt provide clear and valuable guidance for the subsequent optimization process.

\textbf{Low Computation Complexity.}
GiFt excels in efficiency through sparse matrix multiplication. Its calculation only involves the normalization of the adjacency matrix, eliminating the need for the time-consuming parameter learning process with back-propagation. Consequently, its computation complexity is $O(e)$, where $e$ represents the number of edges.

% In summary, GF-Placer achieves high-quality placement without incurring training costs, as seen in deep-learning based placers, or the long iterations required by classic optimization-based placers.

\subsection{Analyzing Existing Placers from GSP Perspective}~\label{sec:method3}
This section analyzes the classic eigenvector-based placer and recently emerged GCN-based placer from the GSP perspective. It is proved that both of them are special cases of the proposed approach with different graph filters.

\subsubsection{Eigenvector-based Placer}
The classic eigenvector-based placer~\cite{eigen_placer} assumes that all cells are movable. They use the eigenvectors corresponding to the second and third smallest eigenvalues as cell locations.
From the GSP perspective, it is an ideal low-pass filter that only saves the lowest frequencies. The corresponding filter function can be defined as 

\begin{equation}\label{eq:ideal_filter}
h(\lambda_i)=\left\{
\begin{aligned}
 1,\ &\text{if} \ 1<i <=3\\
 0,\ &\text{otherwise} \\
\end{aligned}
\right.
\end{equation}
As eigenvectors (graph signals in the spectrum domain) are directly employed to represent cell locations, there is no need to transform them into the spatial domain using inverse GFT $U$.
Therefore, the filtered signal can be calculated by 
\begin{equation}
g^{\prime}=diag(0, 1, 1, 0, \ldots, 0)U=[u_2, u_3]
\end{equation}

Although this method achieves signal smoothness, it incurs high computational costs when applied to large-scale circuits due to the necessity of eigendecomposition. Moreover, as demonstrated in Section~\ref{sec:method1}, confining graph signals to merely the lowest-frequency components leads to the loss of valuable information present in higher-frequency components, yielding suboptimal solutions.

\subsubsection{GCN-based Placer}
To enhance the placement process, many recent studies have recruited GCNs~\cite{graphplanner, CY2021} for help. These studies apply GCNs to encode connectivity information and generate node embeddings for subsequent calculations.
Node embeddings (denoted as $g^{\prime}$) are computed through GCNs using the formula $g^{\prime}=\mathcal{F}(D^{-\frac{1}{2}}AD^{-\frac{1}{2}}XW)$, where $\mathcal{F}$ denotes the activation function, and the kernel is the normalized adjacency matrix $\tilde{A}$.
As proved in Theorem~\ref{proof:adj_filter}, $\tilde{A}$ corresponds to the filter function $h(\lambda)=1-\lambda$, which can be considered a special case of GiFt.

However, GCNs introduce the learnable parameter $W$, which requires a time-consuming training process via back-propagation. This training overhead is redundant for efficient chip placement. The underlying reason is that the statistical characteristics of graph structures can vary significantly across different circuit designs, and the presence of fixed cells (e.g., fixed IOs) exacerbates this issue. Consequently, it is challenging to train a set of fixed parameters that would be effective across a range of diverse designs.
On the other hand, by carefully designing the graph filter, it can efficiently remove high-frequency noises and produce optimized cell locations based on multi-resolution smooth signals. As demonstrated in Table~\ref{tab:placement_result}, our proposed approach without the need for model training, achieves competitive or superior performance compared to GCNs, providing further evidence in support of our assertion.

\section{Experimental Results} \label{sec:experiment}
This section evaluates the performance of GiFt using the academic benchmark suite ISPD2014 and three real-world industrial designs.

\subsection{Experiment Settings}
GiFt is implemented in Python with PyTorch. 
The coefficients $\alpha_0$, $\alpha_1$ and $\alpha_2$ in Eq.~\ref{eq:proposed_filter} are set to 0.1, 0.7 and 0.2, respectively.
We integrate GiFt with two state-of-the-art placers, RePlAce~\cite{replace} and GPU-accelerated DREAMPlace~\cite{dreamplace3}, to construct GiFt-RePlAce and GiFt-DREAMPlace placement flows.
The integration involves using GiFt-generated cell locations as starting points for movable cells in both RePlAce and DREAMPlace. These placers then complete the placement process, resolving overlaps and producing legalized placement results.
To show the efficacy of GiFt, we evaluate HPWL, total runtime, and the number of placement iterations on ISPD2014 benchmark suite~\cite{ispd2014} (see Table~\ref{tab:placement_result}). 
We also compare GiFt-DREAMPlace against several other competitors, including eigenvector-based initial placement~\cite{eigen_placer}-DREAMPlace (EI-DREAMPlace for short), an initial placement (using the eigenvectors corresponding to the second and third smallest eigenvalues) and placement flow, and GraphPlanner-DREAMPlace~\cite{graphplanner}, a state-of-the-art integrated GCN-based floorplanning and placement flow. 
RePlAce and GiFt-RePlAce are performed on a workstation with Intel i7-7700 3.6GHz CPU and 16GB memory. All other experiments are conducted on a Linux server with 16-core Inter Xeon Gold 6226R @ 2.9GHz and NVIDIA 2080 Ti GPU.

\begin{table*}[]
\caption{Experimental results on ISPD2014 benchmarks. The values in parentheses represent the time spent by GiFt. RePlAce and GiFt-RePlAce operate on CPUs. DREAMPlace and GiFt-DREAMPlace run on GPUs.} \label{tab:placement_result}
\centering
\scriptsize
% \vspace{-1em}
\resizebox{\textwidth}{17mm}{
\begin{tabular}{|c||c|c|ccc|ccc||ccc|ccc|}
\hline
\                            &                          &                        & \multicolumn{3}{c|}{RePlAce}                                                                         & \multicolumn{3}{c|}{GiFt-RePlAce }                                                               & \multicolumn{3}{c|}{DREAMPlace}                                                                               & \multicolumn{3}{c|}{GiFt-DREAMPlace}                                                            \\ \cline{4-15} 
\multirow{-2}{*}{Benchmark} & \multirow{-2}{*}{\#Cells} & \multirow{-2}{*}{\#Net}                           & \multicolumn{1}{c|}{\begin{tabular}[c]{@{}c@{}}HPWL \\ (10\textasciicircum{}6)\end{tabular}} & \multicolumn{1}{c|}{Iteration} & \begin{tabular}[c]{@{}c@{}}Runtime\\ (s)\end{tabular} & \multicolumn{1}{c|}{\begin{tabular}[c]{@{}c@{}}HPWL\\ (10\textasciicircum{}6)\end{tabular}} & \multicolumn{1}{c|}{Iteration}    & \begin{tabular}[c]{@{}c@{}}Runtime\\ (s)\end{tabular} & \multicolumn{1}{c|}{\begin{tabular}[c]{@{}c@{}}HPWL\\ (10\textasciicircum{}6)\end{tabular}} & \multicolumn{1}{c|}{Iteration} & \begin{tabular}[c]{@{}c@{}}Runtime\\ (s)\end{tabular} & \multicolumn{1}{c|}{\begin{tabular}[c]{@{}c@{}}HPWL\\ (10\textasciicircum{}6)\end{tabular}} & \multicolumn{1}{c|}{Iteration}    & \begin{tabular}[c]{@{}c@{}}Runtime\\ (s)\end{tabular} \\ \hline
mgc\_des\_perf\_1          & 112644                   & 112878                 & \multicolumn{1}{c|}{\textbf{5.43}} & \multicolumn{1}{c|}{492}       & 78      & \multicolumn{1}{c|}{5.45}           & \multicolumn{1}{c|}{\textbf{351}} & \textbf{52 (1.79)} & \multicolumn{1}{c|}{\textbf{5.54}}  & \multicolumn{1}{c|}{568}       & 52      & \multicolumn{1}{c|}{5.66}           & \multicolumn{1}{c|}{\textbf{433}} & \textbf{31 (0.54)} \\ \hline
mgc\_des\_perf\_2          & 112644                   & 112878                 & \multicolumn{1}{c|}{\textbf{5.56}} & \multicolumn{1}{c|}{489}       & 82      & \multicolumn{1}{c|}{5.71}           & \multicolumn{1}{c|}{\textbf{351}} & \textbf{56 (1.77)} & \multicolumn{1}{c|}{\textbf{5.96}}  & \multicolumn{1}{c|}{579}       & 50      & \multicolumn{1}{c|}{6.23}           & \multicolumn{1}{c|}{\textbf{447}} & \textbf{35 (0.25)} \\ \hline
mgc\_edit\_dist\_1         & 130661                   & 133223                 & \multicolumn{1}{c|}{13.92}         & \multicolumn{1}{c|}{491}       & 140     & \multicolumn{1}{c|}{\textbf{13.92}} & \multicolumn{1}{c|}{\textbf{346}} & \textbf{96 (3.93)} & \multicolumn{1}{c|}{\textbf{14.24}} & \multicolumn{1}{c|}{614}       & 64      & \multicolumn{1}{c|}{14.25}          & \multicolumn{1}{c|}{\textbf{438}} & \textbf{47 (0.43)} \\ \hline
mgc\_edit\_dist\_2         & 130661                   & 133223                 & \multicolumn{1}{c|}{13.67}         & \multicolumn{1}{c|}{495}       & 141     & \multicolumn{1}{c|}{\textbf{13.66}} & \multicolumn{1}{c|}{\textbf{348}} & \textbf{93 (3.71)} & \multicolumn{1}{c|}{13.98}          & \multicolumn{1}{c|}{658}       & 67      & \multicolumn{1}{c|}{\textbf{13.97}} & \multicolumn{1}{c|}{\textbf{445}} & \textbf{49 (0.44)} \\ \hline
mgc\_fft                   & 32281                    & 33307                  & \multicolumn{1}{c|}{1.93}          & \multicolumn{1}{c|}{398}       & 52      & \multicolumn{1}{c|}{\textbf{1.93}}  & \multicolumn{1}{c|}{\textbf{325}} & \textbf{30 (3.75)} & \multicolumn{1}{c|}{1.96}           & \multicolumn{1}{c|}{579}       & 32      & \multicolumn{1}{c|}{\textbf{1.95}}  & \multicolumn{1}{c|}{\textbf{460}} & \textbf{23 (0.43)} \\ \hline
mgc\_matrix\_mult          & 155325                   & 158527                 & \multicolumn{1}{c|}{10.12}         & \multicolumn{1}{c|}{440}       & 103     & \multicolumn{1}{c|}{\textbf{10.07}} & \multicolumn{1}{c|}{\textbf{330}} & \textbf{72 (1.82)} & \multicolumn{1}{c|}{10.43}          & \multicolumn{1}{c|}{628}       & 62      & \multicolumn{1}{c|}{\textbf{10.39}} & \multicolumn{1}{c|}{\textbf{496}} & \textbf{43 (0.17)} \\ \hline
mgc\_pci\_bridge32\_1      & 30675                    & 30835                  & \multicolumn{1}{c|}{0.95}          & \multicolumn{1}{c|}{448}       & 49      & \multicolumn{1}{c|}{\textbf{0.94}}  & \multicolumn{1}{c|}{\textbf{349}} & \textbf{32 (1.45)} & \multicolumn{1}{c|}{1.13}           & \multicolumn{1}{c|}{742}       & 33      & \multicolumn{1}{c|}{\textbf{1.01}}  & \multicolumn{1}{c|}{\textbf{529}} & \textbf{22 (0.11)} \\ \hline
mgc\_pci\_bridge32\_2      & 30675                    & 30835                  & \multicolumn{1}{c|}{0.97}          & \multicolumn{1}{c|}{450}       & 50      & \multicolumn{1}{c|}{\textbf{0.96}}  & \multicolumn{1}{c|}{\textbf{349}} & \textbf{32 (1.05)} & \multicolumn{1}{c|}{1.02}           & \multicolumn{1}{c|}{623}       & 32      & \multicolumn{1}{c|}{\textbf{1.02}}  & \multicolumn{1}{c|}{\textbf{502}} & \textbf{19 (1.08)} \\ \hline
ratio                      &                          &                        & \multicolumn{1}{c|}{1.00}          & \multicolumn{1}{c|}{1.35}      & 1.50    & \multicolumn{1}{c|}{1.00}           & \multicolumn{1}{c|}{1.00}         & 1.00               & \multicolumn{1}{c|}{1.00}           & \multicolumn{1}{c|}{1.33}      & 1.46    & \multicolumn{1}{c|}{1.00}           & \multicolumn{1}{c|}{1.00}         & 1.00               \\ \hline

\end{tabular}}
% \vspace{-1em}
\end{table*}

\begin{table*}[]
\caption{Experimental results on ISPD2014 benchmarks. In the ``EI-DREAMPlace'' column, the values in parentheses represent the time taken for eigendecomposition. In the ``GraphPlanner-DREAMPlace'' column, the values in parentheses represent model training time. In the ``GiFt-DREAMPlace'' column, the values in parentheses represent the time spent by GiFt.} \label{tab:placement_result2}
\centering
\scriptsize
% \vspace{-1em}
\begin{tabular}{|c||ccc|ccc|ccc|}
\hline
\multirow{2}{*}{Benchmark} & \multicolumn{3}{c|}{EI-DREAMPlace}                                                & \multicolumn{3}{c|}{GraphPlanner-DREAMPlace}                                                                                   & \multicolumn{3}{c|}{GiFt-DREAMPlace}                                                           \\ \cline{2-10} 
                           & \multicolumn{1}{c|}{HPWL (10\textasciicircum{}6)} & \multicolumn{1}{c|}{Iteration} & Runtime   (s) & \multicolumn{1}{c|}{HPWL   (10\textasciicircum{}6)} & \multicolumn{1}{c|}{Iteration} & Runtime   (s)                                            & \multicolumn{1}{c|}{HPWL   (10\textasciicircum{}6)}  & \multicolumn{1}{c|}{Iteration}    & Runtime   (s)        \\ \hline
mgc\_des\_perf\_1     & \multicolumn{1}{c|}{\textbf{5.65}}  & \multicolumn{1}{c|}{481}       & 57 (7)  & \multicolumn{1}{c|}{5.93}  & \multicolumn{1}{c|}{454}       & 33            & \multicolumn{1}{c|}{5.66}           & \multicolumn{1}{c|}{\textbf{433}} & \textbf{31 (0.54)} \\ \hline
mgc\_des\_perf\_2     & \multicolumn{1}{c|}{\textbf{5.99}}  & \multicolumn{1}{c|}{495}       & 59 (7)  & \multicolumn{1}{c|}{6.11}  & \multicolumn{1}{c|}{460}       & 30            & \multicolumn{1}{c|}{6.23}           & \multicolumn{1}{c|}{\textbf{447}} & \textbf{35 (0.25)} \\ \hline
mgc\_edit\_dist\_1    & \multicolumn{1}{c|}{\textbf{14.26}} & \multicolumn{1}{c|}{473}       & 66 (4)  & \multicolumn{1}{c|}{14.41} & \multicolumn{1}{c|}{500}       & 44            & \multicolumn{1}{c|}{14.25}          & \multicolumn{1}{c|}{\textbf{438}} & \textbf{47 (0.43)} \\ \hline
mgc\_edit\_dist\_2    & \multicolumn{1}{c|}{13.97}          & \multicolumn{1}{c|}{490}       & 68 (4)  & \multicolumn{1}{c|}{14.11} & \multicolumn{1}{c|}{491}       & 54            & \multicolumn{1}{c|}{\textbf{13.97}} & \multicolumn{1}{c|}{\textbf{445}} & \textbf{49 (0.44)} \\ \hline
mgc\_fft              & \multicolumn{1}{c|}{1.96}           & \multicolumn{1}{c|}{498}       & 52 (15) & \multicolumn{1}{c|}{2.02}  & \multicolumn{1}{c|}{462}       & 22            & \multicolumn{1}{c|}{\textbf{1.95}}  & \multicolumn{1}{c|}{\textbf{460}} & \textbf{23 (0.43)} \\ \hline
mgc\_matrix\_mult     & \multicolumn{1}{c|}{10.41}          & \multicolumn{1}{c|}{539}       & 59 (2)  & \multicolumn{1}{c|}{10.52} & \multicolumn{1}{c|}{498}       & 41            & \multicolumn{1}{c|}{\textbf{10.39}} & \multicolumn{1}{c|}{\textbf{496}} & \textbf{43 (0.17)} \\ \hline
mgc\_pci\_bridge32\_1 & \multicolumn{1}{c|}{1.14}           & \multicolumn{1}{c|}{535}       & 28 (2)  & \multicolumn{1}{c|}{1.08}  & \multicolumn{1}{c|}{604}       & 26            & \multicolumn{1}{c|}{\textbf{1.01}}  & \multicolumn{1}{c|}{\textbf{529}} & \textbf{22 (0.11)} \\ \hline
mgc\_pci\_bridge32\_2 & \multicolumn{1}{c|}{1.15}           & \multicolumn{1}{c|}{533}       & 28 (2)  & \multicolumn{1}{c|}{1.03}  & \multicolumn{1}{c|}{511}       & 20            & \multicolumn{1}{c|}{\textbf{1.02}}  & \multicolumn{1}{c|}{\textbf{502}} & \textbf{19 (0.10)} \\ \hline
ratio                 & \multicolumn{1}{c|}{1.001}          & \multicolumn{1}{c|}{1.09}      & 1.56    & \multicolumn{1}{c|}{1.01}  & \multicolumn{1}{c|}{1.06}      & 3.14 (+0.16h) & \multicolumn{1}{c|}{1.00}           & \multicolumn{1}{c|}{1.00}         & 1.00               \\ \hline
                          
\end{tabular}
\vspace{-1em}
\end{table*}

\subsection{Performance}

\subsubsection{Comparison to the state-of-the-art placers}
Table~\ref{tab:placement_result} presents the placement results using ISPD2014 benchmarks.
The execution time of GiFt is less than 5 seconds on the CPU platform and less than 1 second on the GPU platform. 
From the result we can see that with similar HPWL, GiFt-RePlAce shows 35\% and 50\% reductions in the number of iterations and total runtime compared to RePlAce. Moreover, compared with GPU-accelerated DREAMPlace, GiFt-DREAMPlace achieves a 33\% reduction in the average number of placement iterations and a 46\% reduction in total runtime.
These results demonstrate GiFt's capability to significantly accelerate the placement process.

Figure~\ref{fig:fig3} (a) presents cell locations produced by GiFt on mgc\_edit\_dist\_2 benchmark. To investigate the effect of GiFt, we compare the density curves during placement with and without GiFt. As illustrated in Figure~\ref{fig:fig3} (b), GiFt can effectively drive the placer to bypass the time-consuming initialization process and complete placement with significantly fewer iterations. Owing to the high efficiency of GiFt, the total runtime of the placement process is substantially reduced.

\begin{figure}[]
\includegraphics[width=\linewidth]{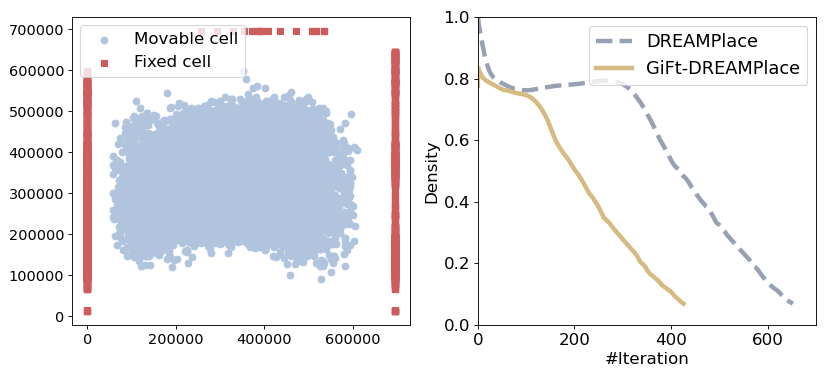}
\centering  
% \vspace{-2em}
\caption{(a) Cell locations produced by GiFt. (b) Density curves during placement.}\label{fig:fig3}
% \vspace{-2em}
\end{figure}

\subsubsection{Comparison to other initial placement strategies.}
Table~\ref{tab:placement_result2} presents the placement results with different initial placement-placement flows. The GPU version of DREAMPlace is used as the subsequent placer in this experiment. Compared with EI-DREAMPlace and GraphPlanner-DREAMPlace, GiFt-DREAMPlace achieves 0.1\% and 1\% improvements in HPWL, 9\% and 6\% reductions in the number of iterations, and 56\% and 3x reductions in total runtime. This result reinforces the findings presented in Section~\ref{sec:method3} and further proves the efficacy of GiFt in accelerating the placement process.
 
\section{Discussion} \label{sec:discussion}

This section provide the theoretical foundations about why GiFt can drive the gradient descent-based method (e.g. DREAMPlace) to efficiently solve the placement problem.

The placement problem is a non-convex nonlinear optimization problem. Gradient-descent methods such as DREAMPlace and RePlAce require a large number of optimization iterations and are prone to falling into local optima. A critical factor influencing the solution quality of these gradient-descent methods is their dependency on the initial solution~\cite{eplace_ms}. 
Our GSP-based method provides an optimized starting point for subsequent placers, that is, it allows gradient descent-based placers to begin the optimization process from a more advantageous position. As a result, the placement problem can be tackled more effectively, leading to high-quality solutions with a significantly reduced number of iterations.

\section{Conclusion} \label{sec:conclusion}

This paper presents the power of GSP in accelerating chip placement. By underscoring the significance of signal smoothness in addressing placement problems, we propose GiFt, a highly efficient placement speedup technique designed to leverage multi-resolution smooth signals for the generation of high-quality initial placement solutions. By integrating GiFt with analytical placers, GiFt-Placer substantially reduces placement optimization iterations without the need for time-consuming model training. 
Experimental results demonstrate that GiFt-Placer consistently achieves competitive or superior performance compared to state-of-the-art placers. In particular, it significantly enhances placement efficiency and even outperforms analytical placers running on GPU versions.

\begin{acks}
  This research is supported in part by the National Natural Science Foundation of China under Grant 62090025 and Grant 92373207; and in part by the National Key Research and Development Program of China under Grant 2023YFB4405101, Grant 2022YFB4400400, Grant 2023YFB4405103, and Grant 2023YFB4405104.
\end{acks}

\bibliographystyle{ACM-Reference-Format}
\bibliography{refs}

\end{document}